\documentclass[11pt]{article}

\pdfoutput=1

\usepackage{natbib}

\usepackage[T1]{fontenc}
\usepackage[latin9]{inputenc}
\usepackage{float}
\usepackage{fullpage,times,color}
\usepackage{boxedminipage}

\usepackage{amsmath,amsthm,amssymb}

\usepackage{graphicx}
\usepackage{ltxtable} 
\newcolumntype{L}{>{\centering\arraybackslash} m{0.04\columnwidth}} 
\newcolumntype{R}{>{\centering\arraybackslash} m{0.48\columnwidth}} 
\newcolumntype{S}{>{\centering\arraybackslash} m{0.32\columnwidth}} 

\newcommand{\tE}{{\mathbb E}_t}

\newtheorem{lemma}{Lemma}

\newtheorem{theorem}{Theorem}

\DeclareMathOperator*{\E}{\mathbb{E}}

\DeclareMathOperator*{\argmin}{argmin} 
 
\newcommand{\reals}{\mathbb{R}}

\newcommand{\figref}[1]{Figure~\ref{#1}}

\newcommand{\secref}[1]{Section~\ref{#1}}

\newcommand{\thmref}[1]{Theorem~\ref{#1}}

\newcommand{\lemref}[1]{Lemma~\ref{#1}}

\newcommand{\todo}[1]{\textbf{#1}}

\newenvironment{myalgo}[1]%
{
\begin{center}
\begin{boxedminipage}{0.8\linewidth}
\begin{center}
\textbf{\texttt{#1}}
\end{center}
\rm
\begin{tabbing}
....\=...\=...\=...\=...\=  \+ \kill
} %
{\end{tabbing} 
\end{boxedminipage} \end{center} 
}

\title{Accelerated  Mini-Batch Stochastic Dual Coordinate Ascent}
\author{Shai Shalev-Shwartz\\
School of Computer Science and Engineering \\
Hebrew University, Jerusalem, Israel
 \and 
Tong Zhang \\
Department of Statistics \\
Rutgers University, NJ, USA}
\date{}

\begin{document}

\maketitle

\begin{abstract}
  Stochastic dual coordinate ascent (SDCA) is an effective technique for solving
  regularized loss minimization problems in machine learning.
  This paper considers an extension of SDCA under the mini-batch setting that is often used in practice.
  Our main contribution is to introduce an accelerated mini-batch version of SDCA and prove a fast convergence rate for this method.  
  We discuss an implementation of our method over a parallel computing system, and compare the results to both the
  vanilla stochastic dual coordinate ascent and to the accelerated
  deterministic gradient descent method of \cite{nesterov2007gradient}.
\end{abstract}

\section{Introduction}

We consider the following generic optimization problem. 
Let $\phi_1,\ldots,\phi_n$ be a sequence of vector convex functions from
$\reals^d$ to $\reals$, and let $g : \reals^d \to \reals$ be a
strongly convex regularization function.
Our goal is to solve $\min_{x \in \reals^d} P(x)$ where
\begin{equation} \label{eqn:PrimalProblem}
P(x) = \left[ \frac{1}{n} \sum_{i=1}^n \phi_i(x) +  g(x) \right] .
\end{equation}
For example, given a sequence of $n$ training examples
$(v_1,y_1),\ldots,(v_n,y_n)$, where $v_i \in \reals^d$ and $y_i \in
\reals$, ridge regression is obtained by setting $g(x) =
\frac{\lambda}{2} \|x\|^2$ and $\phi_i(x) = (x^\top v_i-y_i)^2$.
Regularized logistic regression is obtained by setting
$\phi_i(x) = \log(1+\exp(-y_i x^\top v_i))$. 

The \emph{dual} problem of \eqref{eqn:PrimalProblem} is defined as
follows: For each $i$, let $\phi_i^* : \reals^d \to \reals$ be the
convex conjugate of $\phi_i$, namely, $\phi_i^*(u) = \max_{z \in
  \reals^d} (z^\top u - \phi_i(z))$. Similarly, let $g^*$ be the convex
conjugate of $g$. The dual problem is:
\begin{equation} \label{eqn:DualProblem}
\max_{\alpha \in \reals^{d \times n}} D(\alpha) ~~~\textrm{where}~~~ D(\alpha) = 
\left[ \frac{1}{n} \sum_{i=1}^n -\phi_i^*(-\alpha_i) -  g^*\left( \tfrac{1}{n} \sum_{i=1}^n \alpha_i \right) \right] ~,
\end{equation}
where for each $i$, $\alpha_i$ is the $i$'th column of the matrix
$\alpha$. 

The dual objective has a different dual vector associated with each
primal function. Dual Coordinate Ascent (DCA) methods solve the dual
problem iteratively, where at each iteration of DCA, the dual
objective is optimized with respect to a single dual vector, while the
rest of the dual vectors are kept in tact.  Recently,
\cite{ShalevZh2013} analyzed a stochastic version of dual coordinate
ascent, abbreviated by SDCA, in which at each round we choose which
dual vector to optimize uniformly at random. In particular, let $x^*$
be the optimum of \eqref{eqn:PrimalProblem}. We say that a solution
$x$ is $\epsilon$-accurate if $P(x) - P(x^*) \le \epsilon$. \cite{ShalevZh2013} have
derived the following convergence guarantee for SDCA: If $g(x) =
\frac{\lambda}{2}\|x\|_2^2$ and each $\phi_i$ is $\gamma$-smooth, then
for every $\epsilon > 0$, if we run SDCA for at least
\[
\left(n + \tfrac{1}{\lambda \gamma}\right) \, \log( (n +
\tfrac{1}{\lambda \gamma}) \cdot \tfrac{1}{\epsilon})
\]
iterations, then the solution of the SDCA algorithm
will be $\epsilon$-accurate (in expectation). 
This convergence rate is significantly better than the more commonly studied 
stochastic gradient descent (SGD) methods that are related to
SDCA\footnote{An exception is the recent analysis given in
  \cite{LSB12-sgdexp} for a variant of SGD.}.

Another approach to solving \eqref{eqn:PrimalProblem} is deterministic
gradient descent methods. In particular, \cite{nesterov2007gradient}
proposed an accelerated gradient descent (AGD) method for solving
\eqref{eqn:PrimalProblem}. Under the same conditions mentioned above,
AGD finds an $\epsilon$-accurate solution after performing
\[
O\left(\frac{1}{\sqrt{\lambda \gamma}} \log( \tfrac{1}{\epsilon} ) \right)
\]
iterations. 

The advantage of SDCA over AGD is that each iteration involves only a
single dual vector and in general costs $O(d)$. In contrast, each
iteration of AGD requires $\Omega(nd)$ operations. On the other hand,
AGD has a better dependence on the condition number of the problem ---
the iteration bound of AGD scales with $1/\sqrt{\lambda \gamma}$ while
the iteration bound of SDCA scales with $1/(\lambda \gamma)$. 

In this paper we describe and analyze a new algorithm that
interpolates between SDCA and AGD. At each iteration of the algorithm,
we randomly pick a subset of $m$ indices from $\{1,\ldots,n\}$ and
update the dual vectors corresponding to this subset. This subset is
often called a mini-batch. The use of mini-batches is common with SGD
optimization, and it is beneficial when the processing time of a
mini-batch of size $m$ is much smaller than $m$ times the processing
time of one example (mini-batch of size $1$).  For example, in the
practical training of neural networks with SGD, one is always advised
to use mini-batches because it is more efficient to perform
matrix-matrix multiplications over a mini-batch than an equivalent
amount of matrix-vector multiplication operations (each over a single
training example).  This is especially noticeable when GPU is used: in
some cases the processing time of a mini-batch of size 100 may be the
same as that of a mini-batch of size 10.  Another typical use of
mini-batch is for parallel computing, which was studied by various
authors for stochastic gradient descent (e.g.,
\cite{dekel2012optimal}). This is also the application scenario we
have in mind, and will be discussed in greater details in
Section~\ref{sec:parallel}.

Recently, \cite{takac2013mini} studied mini-batch variants of SDCA in
the context of the Support Vector Machine (SVM) problem. They have
shown that the naive mini-batching method, in which $m$ dual variables
are optimized in parallel, might actually increase the number of
iterations required. They then describe several ``safe'' mini-batching
schemes, and based on the analysis of \cite{ShalevZh2013}, have shown
several speed-up results. However, their results are for the
non-smooth case and hence they do not obtain linear convergence
rate. In addition, the speed-up they obtain requires some spectral
properties of the training examples. We take a different approach and
employ Nesterov's acceleration method, which has previously been
applied to mini-batch SGD optimization.  This paper shows how to
achieve acceleration for SDCA in the mini-batch setting.  The pseudo
code of our Accelerated Mini-Batch SDCA, abbreviated by ASDCA, is
presented below.

\begin{myalgo}{Procedure Accelerated Mini-Batch SDCA} 
\textbf{Parameters}  scalars $\lambda,\gamma$ and $\theta \in [0,1]$
~;~ mini-batch size $m$\\
\textbf{Initialize} $\alpha_1^{(0)}=\cdots = \alpha_n^{(0)}=\bar{\alpha}^{(t)}=0$, $x^{(0)}=0$ \\
\textbf{Iterate:} for $t=1,2,\dots$ \+ \\
$u^{(t-1)} = (1-\theta) x^{(t-1)} + \theta \nabla g^*(\bar{\alpha}^{(t-1)})$ \\
 Randomly pick subset $I \subset \{1,\ldots,n\}$ of size $m$ and update the dual variables in $I$ \+\+ \\
$\alpha^{(t)}_i = (1-\theta) \alpha^{(t-1)}_i - \theta \nabla \phi_i(u^{(t-1)})$ for $i \in I$ \\
$\alpha^{(t)}_j = \alpha^{(t-1)}_j$  for $j \notin I$ \- \- \\
$\bar{\alpha}^{(t)}= \bar{\alpha}^{(t-1)} + n^{-1} \sum_{i \in I} (\alpha_i^{(t)}-\alpha_i^{(t-1)})$ \\
$x^{(t)} = (1-\theta) x^{(t-1)} + \theta \nabla g^*(\bar{\alpha}^{(t)})$ \- \\
\textbf{end} 
\end{myalgo}

In the next section we present our main result --- an analysis of the
number of iterations required by ASDCA. We focus on the case of
Euclidean regularization, namely, $g(x) = \frac{\lambda}{2}
\|x\|^2$. Analyzing more general strongly convex regularization
functions is left for future work. In \secref{sec:parallel} we discuss
parallel implementations of ASDCA and compare it to parallel
implementations of AGD and SDCA.  In particular, we explain in which
regimes ASDCA can be better than both AGD and SDCA. In
\secref{sec:experimental} we present some experimental results,
demonstrating how ASDCA interpolates between AGD and SDCA. The proof
of our main theorem is presented in \secref{sec:proof}. We conclude
with a discussion of our work in light of related works in
\secref{sec:related}.

\section{Main Results} \label{sec:main}

Our main result is a bound on the number of iterations required by
ASDCA to find an $\epsilon$-accurate solution. In our analysis, we
only consider the squared Euclidean norm regularization,
\[
g(x) = \frac{\lambda}{2} \|x\|^2 ,
\]
where $\|\cdot\|$ is the Euclidean norm and $\lambda > 0$ is a
regularization parameter. The analysis for general $\lambda$-strongly
convex regularizers is left for future work.  For the squared
Euclidean norm we have
\[
g^*(\alpha) = \frac{1}{2\lambda} \| \alpha\|^2 ,
\]
and
\[
\nabla g^*(\alpha) = \frac{1}{\lambda} \alpha ~.
\]
We further assume that each $\phi_i$ is $1/\gamma$-smooth with respect
to $\|\cdot\|$, namely, 
\[
\forall x,z,~~~
\phi_i(x) \leq \phi_i(z) + \nabla \phi_i(z)^\top (x-z) + \frac{1}{2
  \gamma} \|x-z\|^2 .
\]
For example, if $\phi_i(x) = (x^\top v_i - y_i)^2$, then it is
$\|v_i\|^2$-smooth. 

The smoothness of $\phi_i$ also implies that $\phi_i^*(\alpha)$ is
$\gamma$-strongly convex: 
\[
\forall \theta \in [0,1],~~\phi_i^*((1-\theta) \alpha + \theta \beta) \leq (1-\theta) \phi_i^*(\alpha) + \theta \phi_i^*(\beta) - \frac{\theta (1-\theta) \gamma }{2} \|\alpha-\beta\|^2 ,
\]

\begin{theorem} \label{thm:main} Assume that $g(x) =
  \frac{1}{2\lambda} \|x\|_2^2$ and for each $i$, $\phi_i$ is
  $(1/\gamma)$-smooth w.r.t. the Euclidean norm. Suppose that the
  ASDCA algorithm is run with parameters $\lambda,\gamma,m,\theta$,
  where
\begin{equation} \label{eqn:thetadef}
\theta \le \frac{1}{4} \min\left\{
1 ~,~ \sqrt{\frac{\gamma \lambda n}{m}} ~,~ \gamma \lambda n ~,~
\frac{(\gamma \lambda n)^{2/3}}{m^{1/3}} \right\} ~.
\end{equation}
Define the dual sub-optimality by $\Delta D(\alpha) =
D(\alpha^*)-D(\alpha)$, where $\alpha^*$ is the optimal dual solution,
and the primal sub-optimality by $\Delta P(x) = P(x) -
D(\alpha^*)$. Then,
\[
m \E \Delta P(x^{(t)}) + n \E \Delta D(\alpha^{(t)}) \leq (1-\theta m/n)^t [ m \Delta P(x^{(0)}) + n \Delta D(\alpha^{(0)})] .
\]
It follows that after performing
\[
t \ge  \frac{n/m}{\theta} \, \log\left(\frac{m \Delta P(x^{(0)}) + n \Delta D(\alpha^{(0)})}{m\epsilon}\right) 
\]
iterations, we have that $\E[P(x^{(t)})- D(\alpha^{(t)})] \le \epsilon$. 
\end{theorem}

Let us now discuss the bound, assuming $\theta$ is taken to be the
right-hand side of \eqref{eqn:thetadef}. The dominating factor of the
bound on $t$ becomes
\begin{align}
\frac{n}{m\theta} &=  \frac{n}{m} \cdot \max\left\{
1 ~,~ \sqrt{\frac{m}{\gamma \lambda n}} ~,~ \frac{1}{\gamma \lambda n} ~,~
\frac{m^{1/3}}{(\gamma \lambda n)^{2/3}} \right\} \\
&=\max\left\{
\frac{n}{m}  ~,~ \sqrt{\frac{n/m}{\gamma \lambda }} ~,~ \frac{1/m}{\gamma \lambda} ~,~
\frac{n^{1/3}}{(\gamma \lambda m)^{2/3}} \right\} ~.
\end{align}

Table~\ref{tab:iteration-complexity} summarizes several interesting cases, and compares
the iteration bound of ASDCA to the iteration bound of the vanilla
SDCA algorithm (as analyzed in \cite{ShalevZh2013}) and the Accelerated
Gradient Descent (AGD) algorithm of \cite{nesterov2007gradient}.  In the table, we
ignore constants and logarithmic factors.

\begin{table}
\centering
\begin{tabular}{l|c|c|c} \hline
 Algorithm & $\gamma \lambda n = \Theta(1)$ & $\gamma \lambda n =
 \Theta(1/m)$ & $\gamma \lambda n = \Theta(m)$ \\ \hline
SDCA &  $n$ & $nm$ & $n$\\
ASDCA &  ${n}/{\sqrt{m}}$ & $n$ & ${n}/{m}$ \\
AGD & $\sqrt{n}$ & $\sqrt{nm}$ & $\sqrt{n/m}$\\ \hline
\end{tabular}
\caption{Comparison of Iteration Complexity}
\label{tab:iteration-complexity}
\end{table}

As can be seen in the table, the ASDCA algorithm interpolates between
SDCA and AGD. In particular, ASDCA has the same bound as SDCA when
$m=1$ and the same bound as AGD when $m=n$. Recall that the cost of
each iteration of AGD scales with $n$ while the cost of each iteration
of SDCA does not scale with $n$. The cost of each iteration of ASDCA
scales with $m$. To compensate for the difference cost per iteration for 
different algorithms, we may also compare the complexity in terms of 
the number of examples processed in Table~\ref{tab:example-complexity}.
This is also what we will study in our empirical experiments. 
It should be mentioned that this comparison is meaningful in a single processor
environment, but not in a parallel computing environment when multiple examples can be
processed simultaneiously in a minibatch.
In the next section we discuss under what conditions
the overall runtime of ASDCA is better than both AGD and SDCA.

\begin{table}
\centering
\begin{tabular}{l|c|c|c} \hline
 Algorithm & $\gamma \lambda n = \Theta(1)$ & $\gamma \lambda n =
 \Theta(1/m)$ & $\gamma \lambda n = \Theta(m)$ \\ \hline
SDCA &  $n$ & $nm$ & $n$\\
ASDCA &  $n \sqrt{m}$ & $n m$ & $n$ \\
AGD & $n\sqrt{n}$ & $n\sqrt{nm}$ & $n\sqrt{n/m}$\\ \hline
\end{tabular}
\caption{Comparison of Number of Examples Processed}
\label{tab:example-complexity}
\end{table}






\section{Parallel Implementation}  \label{sec:parallel}

In recent years, there has been a lot of interest in implementing
optimization algorithms using a parallel computing architecture (see
\secref{sec:related}). We now discuss how to implement AGD, SDCA, and
ASDCA when having a computing machine with $s$ parallel computing
nodes.

 In the calculations below, we use the following facts:
\begin{itemize}

\item If each node holds a $d$-dimensional vector, we can compute the
  sum of these vectors in time $O(d\log(s))$ by applying a
  ``tree-structure'' summation (see for example the All-Reduce
  architecture in \cite{agarwal2011reliable}).

\item A node can broadcast a message with $c$ bits to all other nodes
  in time $O(c\log^2(s))$. To see this, order nodes on the corners of
  the $\log_2(s)$-dimensional hypercube.  Then, at each iteration,
  each node sends the message to its $\log(s)$ neighbors (namely, the
  nodes whose code word is at a hamming distance of $1$ from the
  node). The message between the furthest away nodes will pass after
  $\log(s)$ iterations. Overall, we perform $\log(s)$ iterations and
  each iteration requires transmitting $c\log(s)$ bits.

\item All nodes can broadcast a message with $c$ bits to all other
  nodes in time $O(cs \log^2(s))$. To see this, simply apply the
  broadcasting of the different nodes mentioned above in parallel. The
  number of iterations will still be the same, but now, at each
  iteration, each node should transmit $cs$ bits to its $\log(s)$
  neighbors. Therefore, it takes $O(cs \log^2(s))$ time.

\end{itemize}

For concreteness of the discussion, we consider problems in which
$\phi_i(x)$ takes the form of $\ell(x^\top v_i,y_i)$, where $y_i$ is a
scalar and $v_i \in \reals^d$. This is the case in supervised learning
of linear predictors (e.g. logistic regression or ridge
regression). We further assume that the average number of non-zero
elements of $v_i$ is $\bar{d}$. In very large-scale problems, a single
machine cannot hold all of the data in its memory. However, we assume
that a single node can hold a fraction of $1/s$ of the data in its
memory.

Let us now discuss parallel implementations of the different
algorithms starting with deterministic gradient algorithms (such as
AGD). The bottleneck operation of deterministic gradient algorithms is
the calculation of the gradient. In the notation mentioned above, this
amounts to performing order of $n \bar{d}$ operations. If the data is
distributed over $s$ computing nodes, where each node holds $n/s$
examples, we can calculate the gradient in time $O(n\bar{d}/s + d
\log(s))$ as follows. First, each node calculates the gradient over its
own $n/s$ examples (which takes time $O(n\bar{d}/s)$). Then, the $s$
resulting vectors in $\reals^d$ are summed up in time $O(d\log(s))$.

Next, let us consider the SDCA algorithm.  On a single computing node,
it was observed that SDCA is much more efficient than deterministic
gradient descent methods, since each iteration of SDCA costs only
$\Theta(\bar{d})$ while each iteration of AGD costs
$\Theta(n\bar{d})$. When we have $s$ nodes, for the SDCA algorithm,
dividing the examples into $s$ computing nodes does not yield any
speed-up. However, we can divide the \emph{features} into the $s$
nodes (that is, each node will hold $d/s$ of the features for all of
the examples). This enables the computation of $x^\top v_i$ in
(expected) time of $O(\bar{d}/s + s \log^2(s))$. Indeed, node $t$ will
calculate  $\sum_{j \in J_t} x_j v_{i,j}$, where $J_t$ is the set of
features stored in node $t$ (namely, $|J_t|=d/s$). Then, each node broadcasts the
resulting scalar to all the other nodes. Note that we will obtain a
speed-up over the naive implementation only if $s \log^2(s) \ll
\bar{d}$.

For the ASDCA algorithm, each iteration involves the computation of
the gradient over $m$ examples. We can choose to implement it by
dividing the examples to the $s$ nodes (as we did for AGD) or by dividing the features
into the $s$ nodes (as we did for SDCA). In the first case, the cost
of each iteration is $O(m\bar{d}/s + d\log(s))$ while in the latter
case, the cost of each iteration is $O(m\bar{d}/s + ms\log^2(s))$. We
will choose between these two implementations based on the relation
between $d,m,$ and $s$. 

The runtime and communication time of each iteration is summarized in
the table below.

\begin{center}
\begin{tabular}{l|c|c|c}  \hline
 Algorithm & partition type & runtime & communication time \\[0.2cm]
 \hline & & & \\[-0.2cm]
SDCA & features &  ${\bar{d}}/{s} $ &  $s\log^2(s)$  \\[0.2cm]
\textbf{ASDCA} & features  & ${\bar{d}m}/{s} $ & $ms\log^2(s)$  \\[0.2cm]
\textbf{ASDCA} & examples & ${\bar{d}m}/{s} $ & $d\log(s)$  \\[0.2cm]
AGD & examples & ${\bar{d}n}/{s}$ &  $d\log(s)$ \\[0.2cm] \hline
\end{tabular}
\end{center}

We again see that ASDCA nicely interpolates between SDCA and AGD.  In
practice, it is usually the case that there is a non-negligible cost
of opening communication channels between nodes. In that case, it will
be better to apply the ASDCA with a value of $m$ that reflects an
adequate tradeoff between the runtime of each node and the
communication time. With the appropriate value of $m$ (which depends
on constants like the cost of opening communication channels and
sending packets of bits between nodes), ASDCA may outperform both
SDCA and AGD.

\section{Experimental Results} \label{sec:experimental}

In this section we demonstrate how ASDCA interpolates between SDCA and
AGD. All of our experiments are performed for the task of binary
classification with a smooth variant of the hinge-loss (see
\cite{ShalevZh2013}). Specifically, let $(v_1,y_1),\ldots,(v_m,y_m)$
be a set of labeled examples, where for every $i$, $v_i \in \reals^d$
and $y_i \in \{\pm 1\}$. Define $\phi_i(x)$ to be
\[
\phi_i(x) = \begin{cases}
0  & y_i x^\top v_i > 1 \\
1/2-y_i x^\top v_i  & y_i x^\top v_i < 0 \\
\frac{1}{2}(1-y_i x^\top v_i)^2    & \textrm{o.w.}
\end{cases} 
\]
We also set the regularization function to be $g(x) =
\frac{\lambda}{2}\|x\|_2^2$ where $\lambda = 1/n$. This is the default
value for the regularization parameter taken in several optimization
packages.

Following \cite{ShalevZh2013}, the experiments were performed on three
large datasets with very different feature counts and sparsity. The
astro-ph dataset classifies abstracts of papers from the physics ArXiv
according to whether they belong in the astro-physics section; CCAT is
a classification task taken from the Reuters RCV1 collection; and cov1
is class 1 of the covertype dataset of Blackard, Jock \& Dean. The
following table provides details of the dataset characteristics.
\begin{center}
\begin{tabular}{|r|c|c|c|c|}
	\hline
Dataset & Training Size & Testing Size & Features & Sparsity  \\ \hline 
astro-ph & $29882$ & $32487$ & $99757$ & $0.08\%$ \\
CCAT & $781265$ & $23149$ & $47236$ & $0.16\%$ \\
cov1 & $522911$ & $58101$ & $54$ & $22.22\%$ \\
\hline
\end{tabular}
\end{center}

We ran ASDCA with values of $m$ from the set
$\{10^{-4}n,10^{-3}n,10{-2}n\}$. We also ran the SDCA algorithm and
the AGD algorithm. In \figref{fig:smooth} we depict the primal
sub-optimality of the different algorithms as a function of the number
of examples processed. Note that each iteration of SDCA processes a
single example, each iteration of ASDCA processes $m$ examples, and
each iteration of AGD processes $n$ examples. As can be seen from the
graphs, ASDCA indeed interpolates between SDCA and AGD. It is
clear from the graphs that SDCA is much better than AGD when we have a
single computing node. ASDCA performance is quite similar to SDCA when
$m$ is not very large. As discussed in \secref{sec:parallel}, when we
have parallel computing nodes and there is a non-negligible cost of
opening communication channels between nodes, running ASDCA with an
appropriate value of $m$ (which depends on constants like the cost of
opening communication channels) may yield the best performance.

\begin{figure}

\begin{center}
\begin{tabular}{ @{} S @{} S @{} S @{} }
\scriptsize{astro-ph} & \scriptsize{CCAT} & \scriptsize{cov1}\\ \hline
\includegraphics[width=0.31\textwidth]{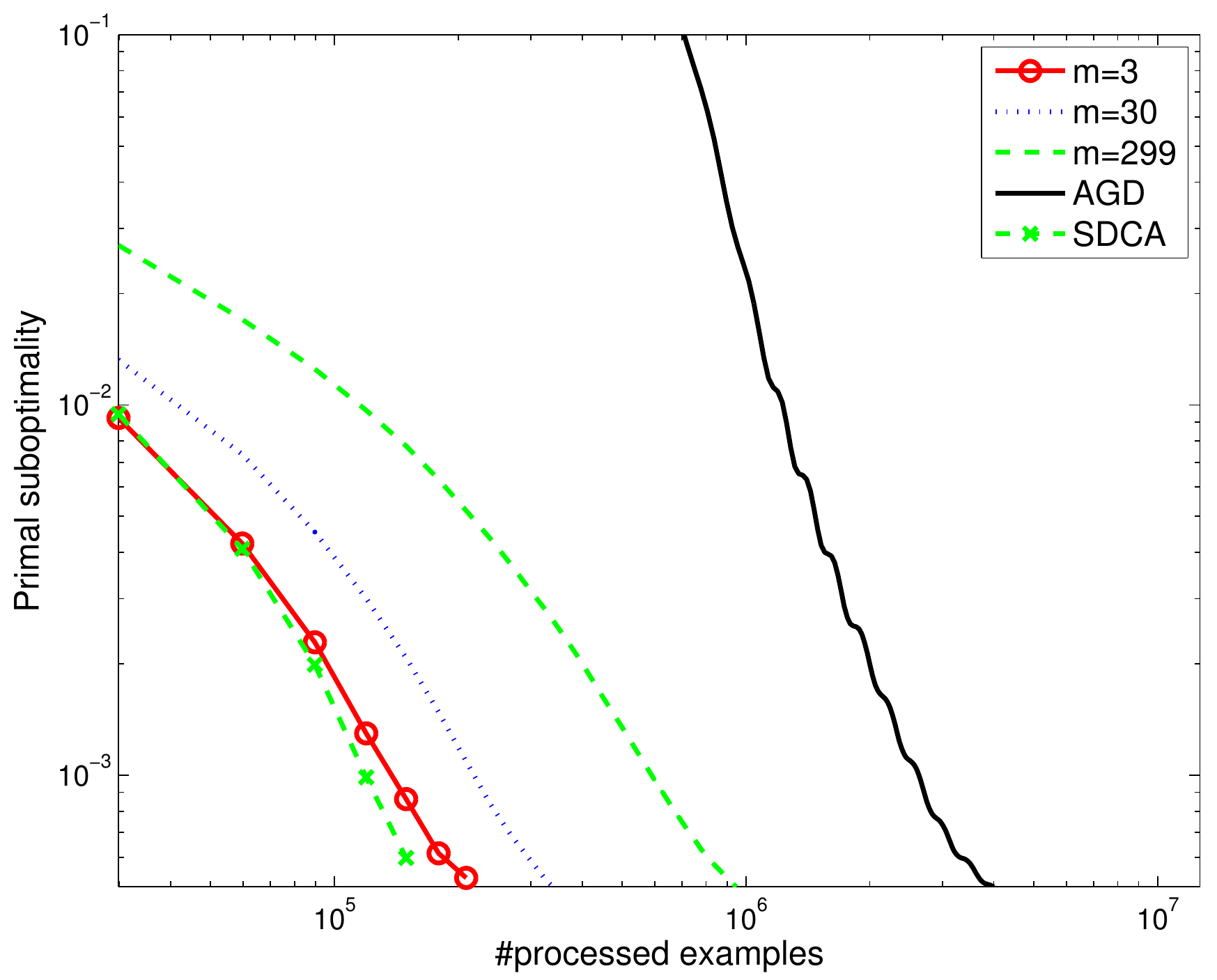} &
\includegraphics[width=0.31\textwidth]{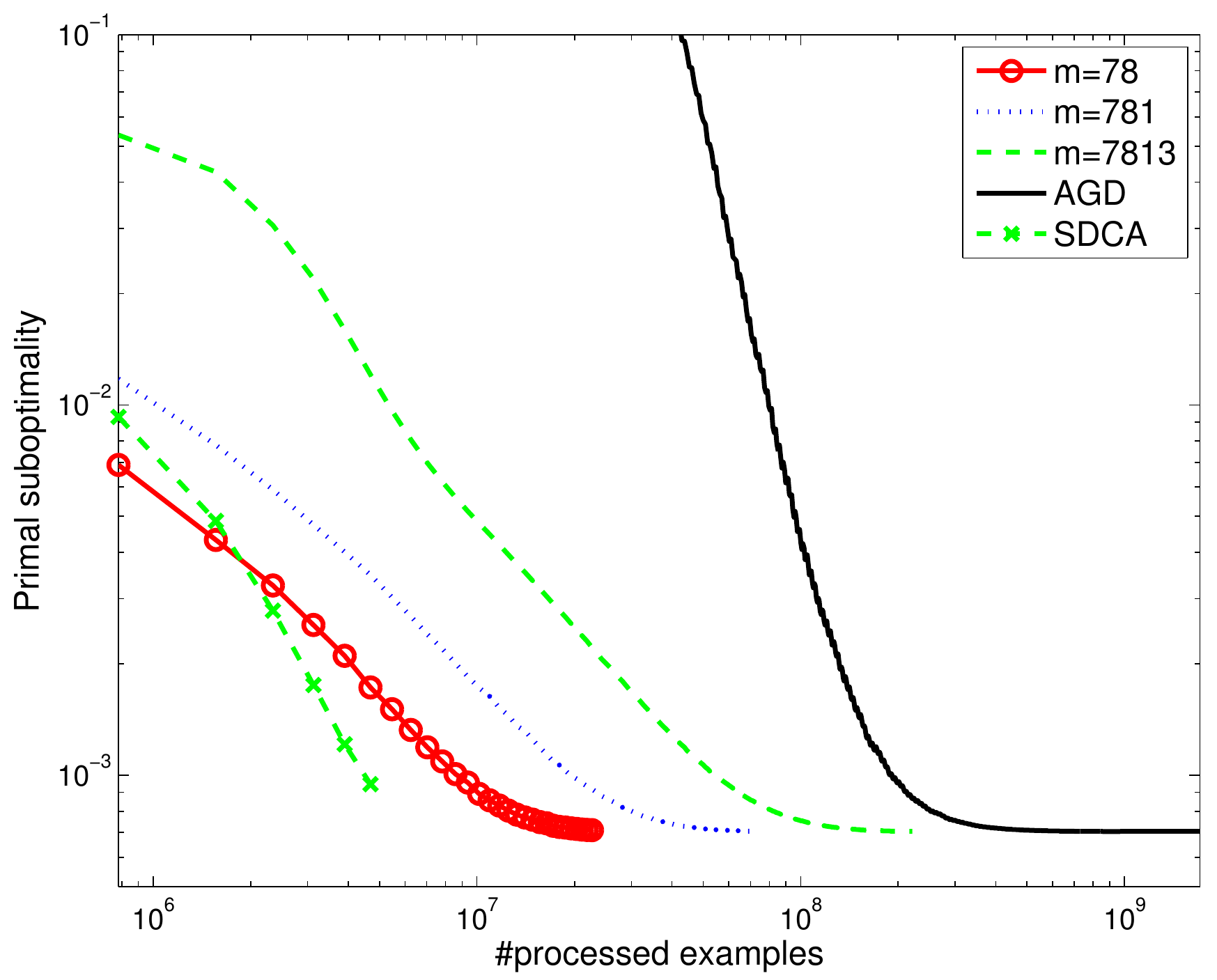}  &
\includegraphics[width=0.31\textwidth]{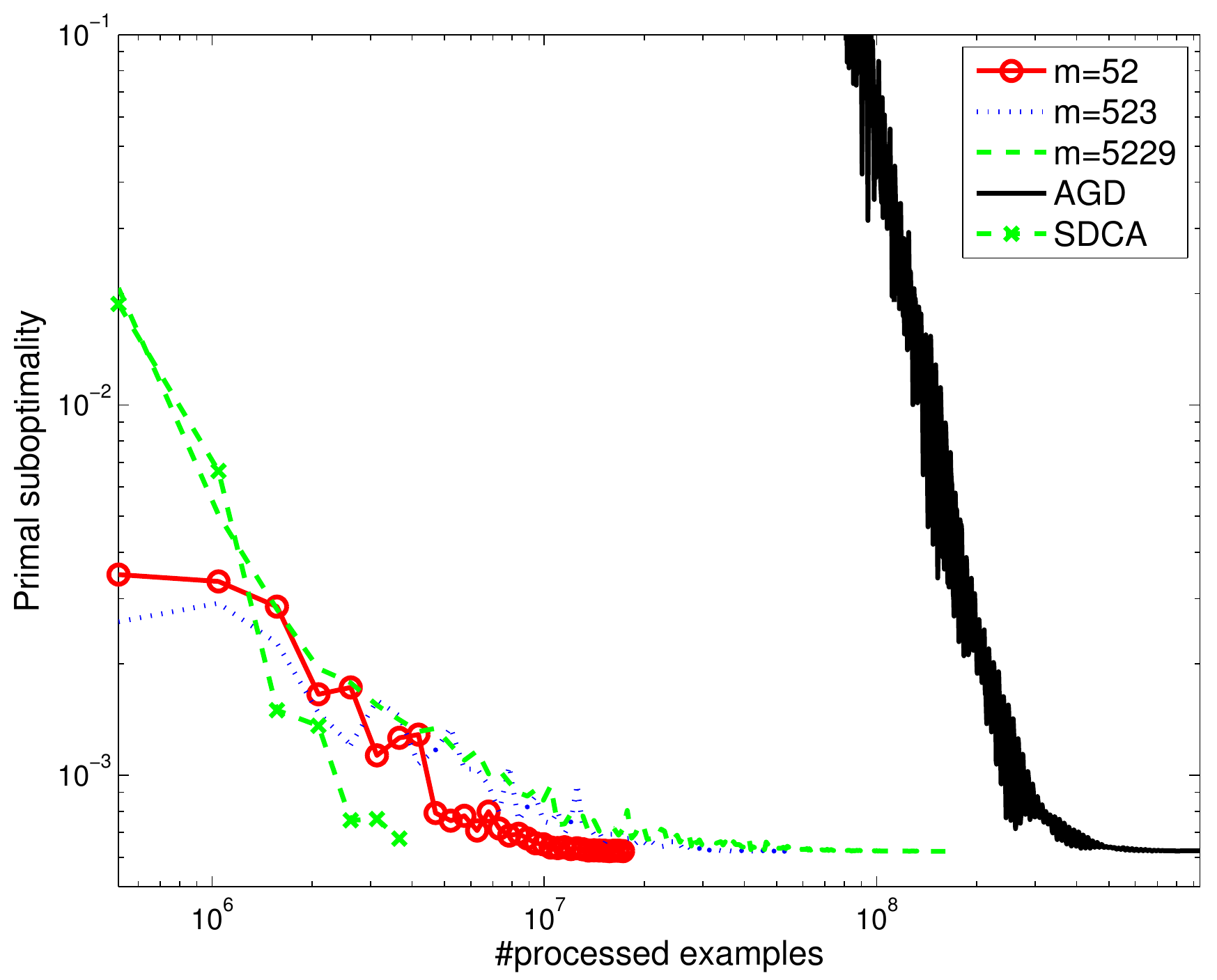}\\
\includegraphics[width=0.31\textwidth]{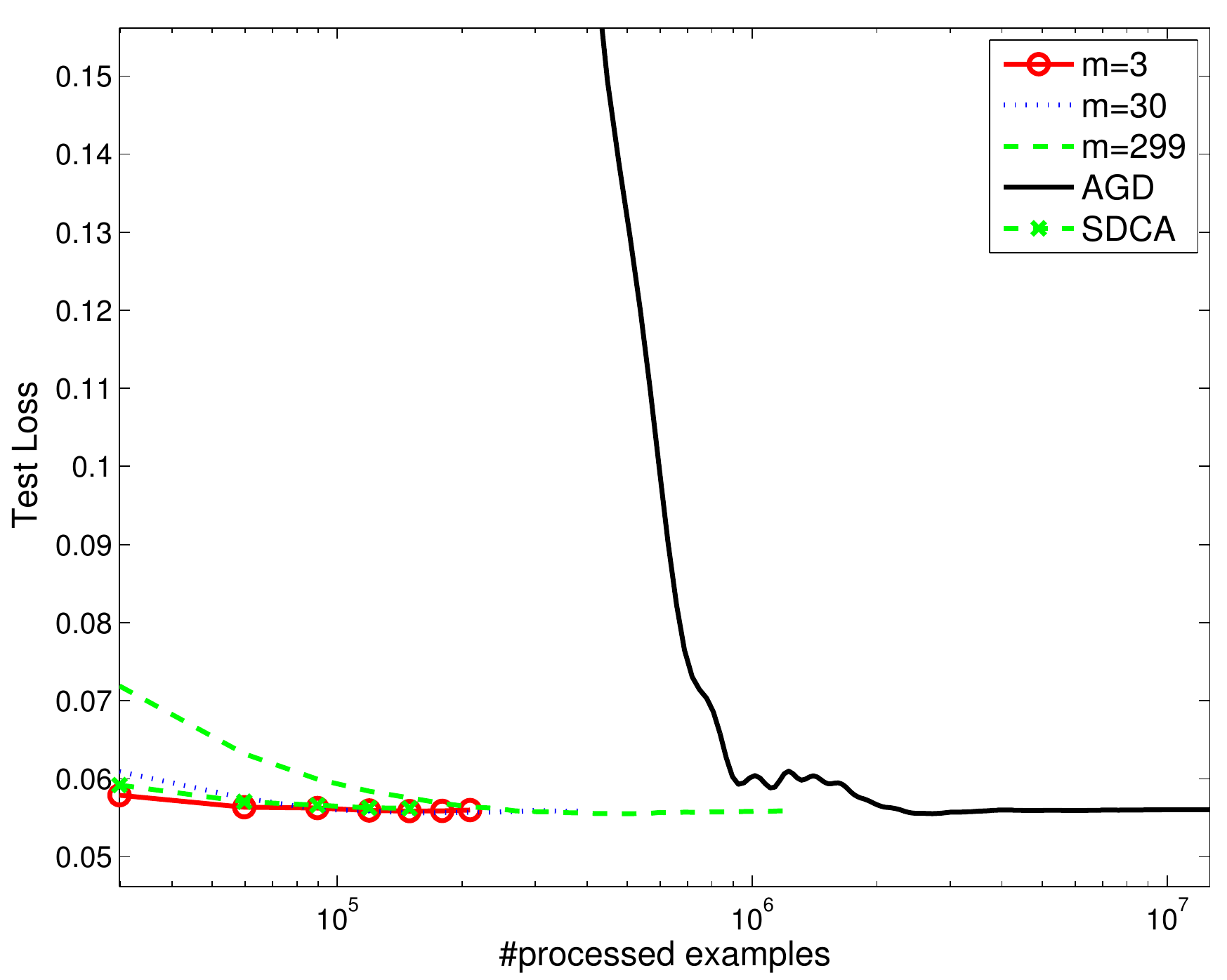} &
\includegraphics[width=0.31\textwidth]{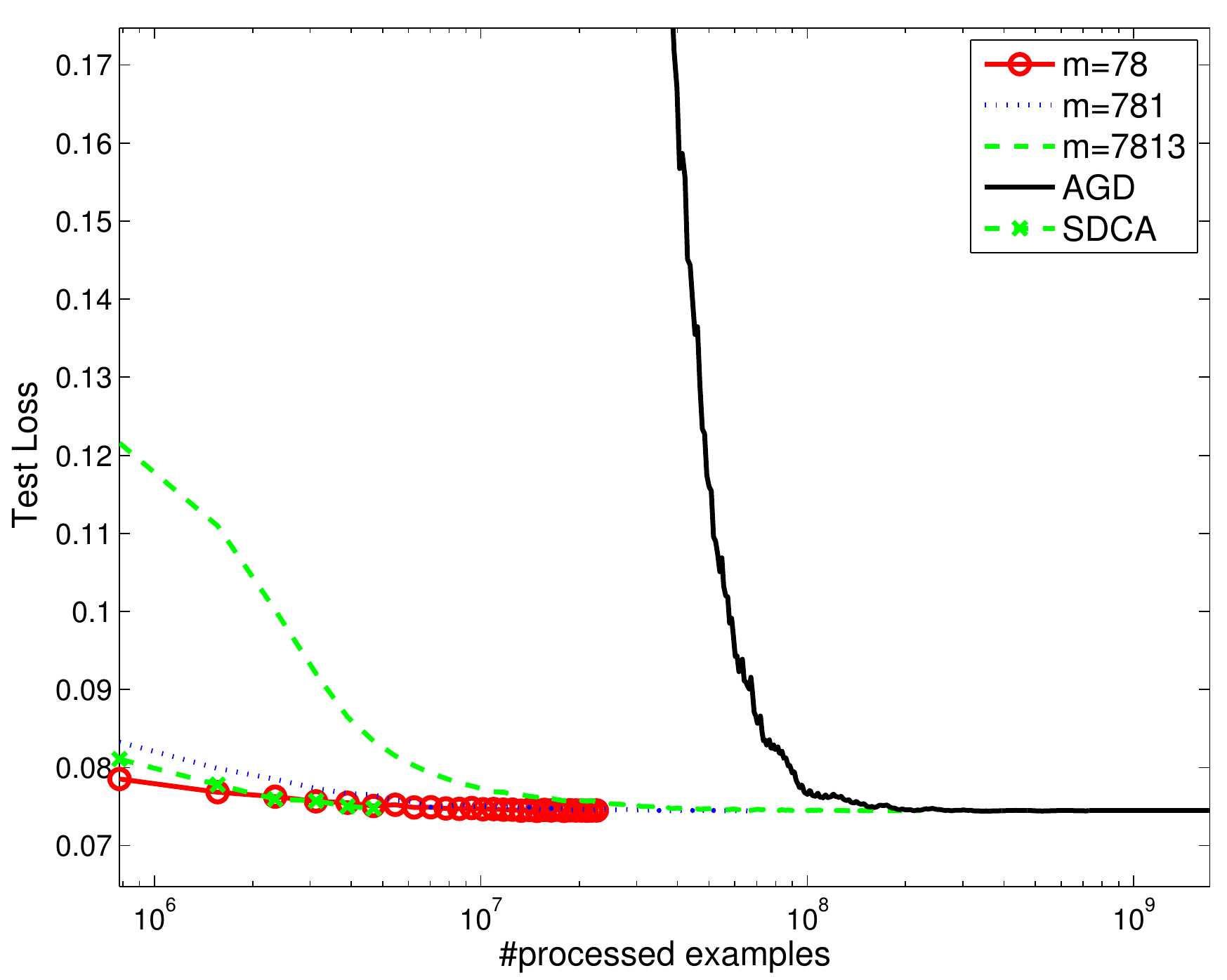}  &
\includegraphics[width=0.31\textwidth]{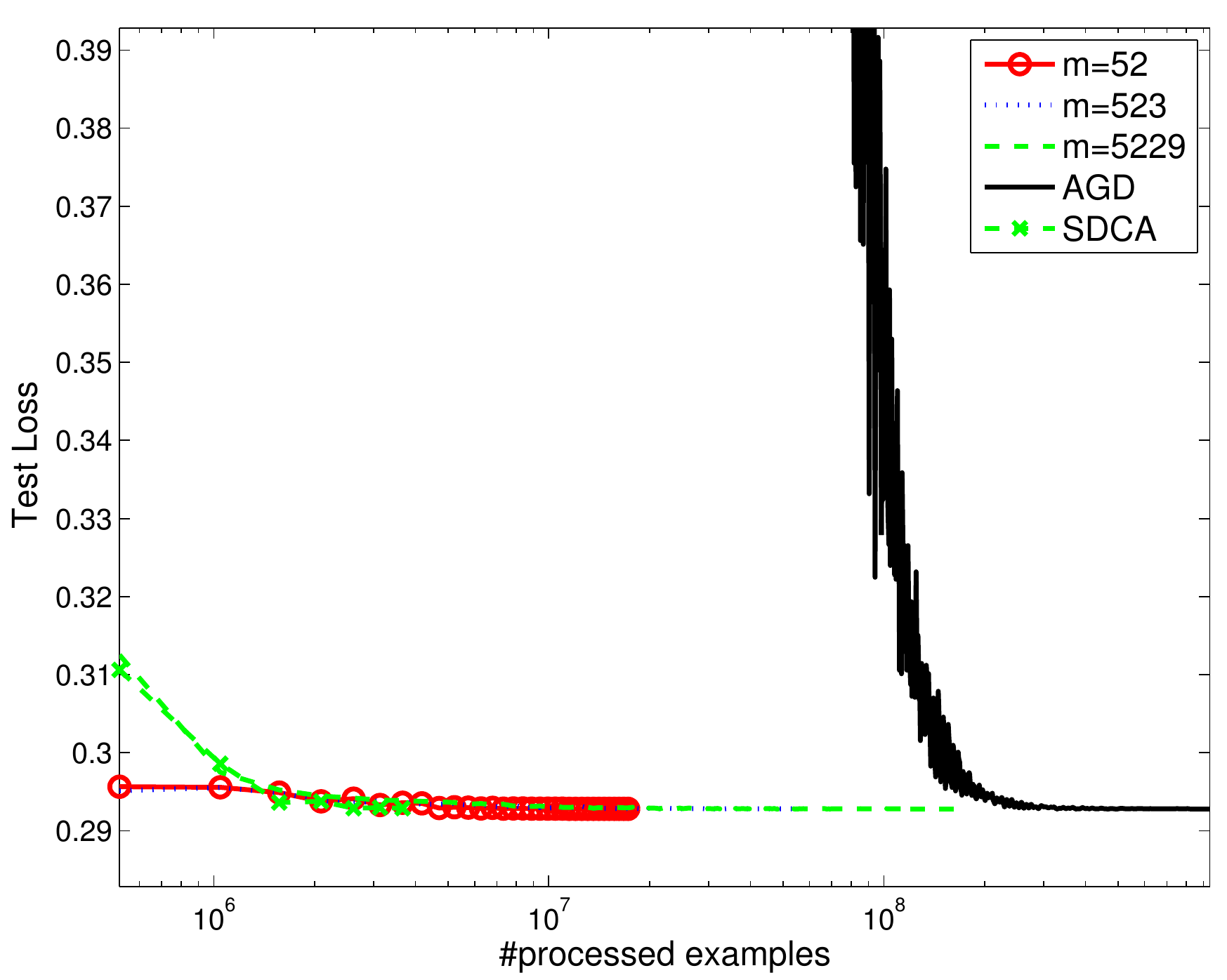}\\
\includegraphics[width=0.31\textwidth]{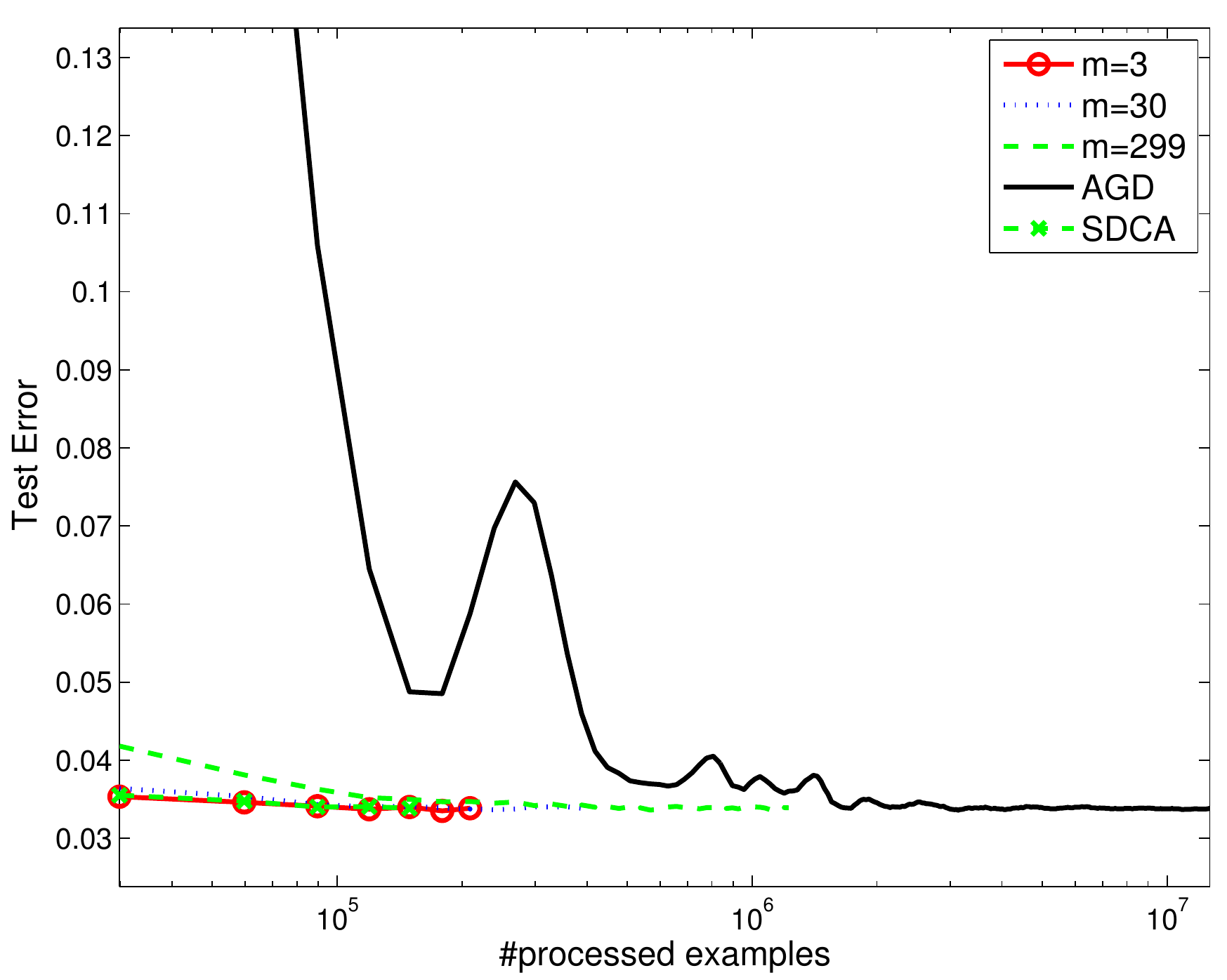} &
\includegraphics[width=0.31\textwidth]{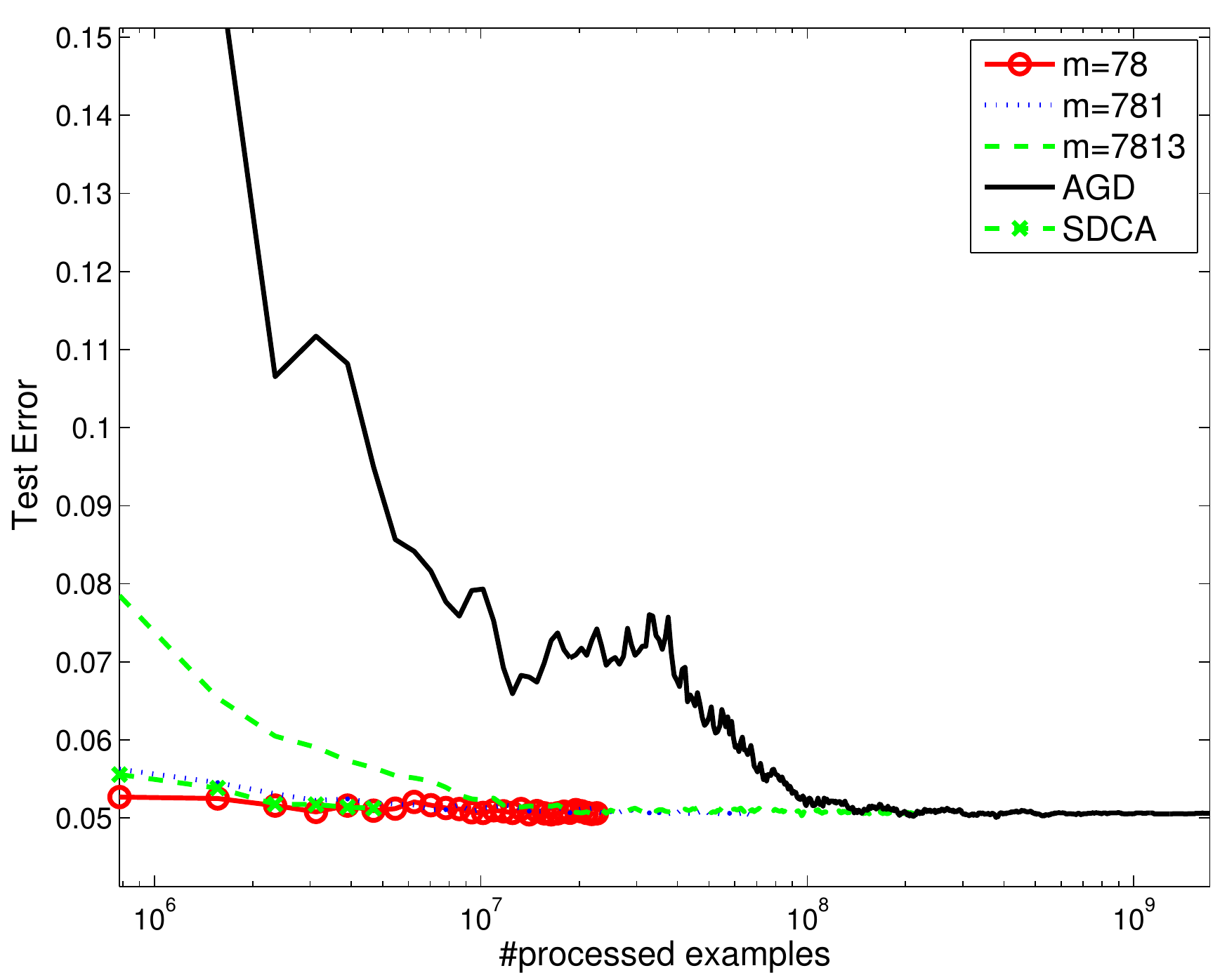}  &
\includegraphics[width=0.31\textwidth]{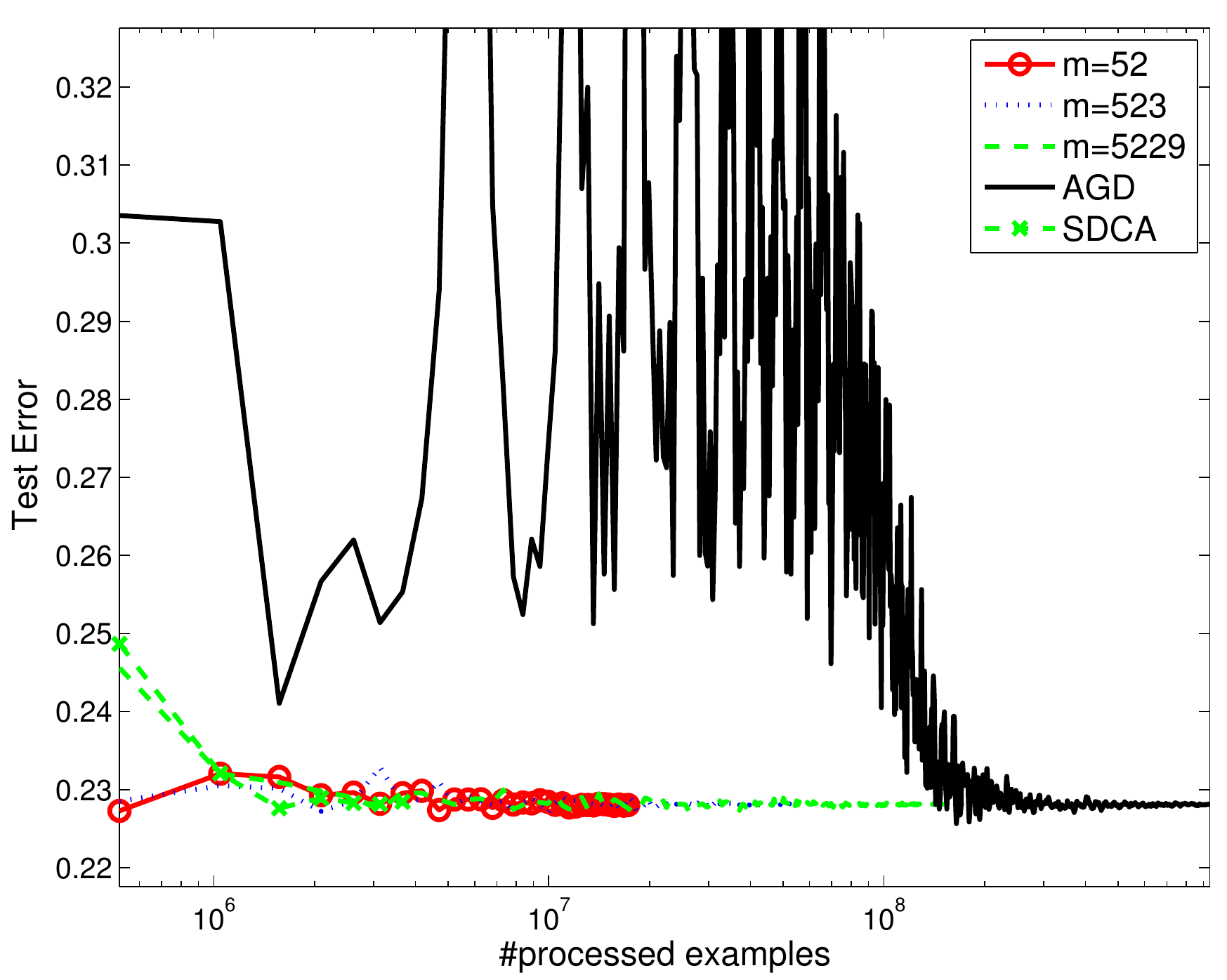}\\
\end{tabular}
\end{center}

\caption{\label{fig:smooth}
The figures presents the performance of AGD, SDCA, and ASDCA with different
values of mini-batch size, $m$. In all figures, the x axis is the
number of processed examples. The three columns are for the different
datasets. Top: primal sub-optimality. Middle: average value of the
smoothed hinge loss function over a test set. Bottom: average value of
the 0-1 loss over a test set.}

\end{figure}

\section{Proof} \label{sec:proof}

We use the following notation:
\begin{align*}
&f(x) = \frac{1}{n} \sum_{i=1}^n \phi_i(x) ~, \\
&\Delta \bar{\alpha}^{(t)} = \bar{\alpha}^{(t)} - \bar{\alpha}^{(t-1)} ~.
\end{align*}
In addition, we use the notation $\tE$ to denote the expectation over
the choice of the set $I$ at iteration $t$, conditioned on the values
of $x^{(t-1)}$ and $\alpha^{(t-1)}$.

Our first lemma calculates the expected value of $\Delta \bar{\alpha}^{(t)}$.
\begin{lemma} \label{lem:DeltatE}
At each round $t$, we have
\[
\tE[\Delta \bar{\alpha}^{(t)}] = - \frac{\theta
  m}{n}\left( \bar{\alpha}^{(t-1)} + \nabla f(u^{(t-1)}) \right) ~.
\]
\end{lemma}
\begin{proof}
By the definition of the update, 
\[
\Delta \bar{\alpha}^{(t)} = \frac{1}{n} \sum_{i \in I} (\alpha_i^{(t)}
  - \alpha_i^{(t-1)}) = \frac{-\theta}{n} \sum_{i \in I}
    (\alpha_i^{(t-1)} + \nabla \phi_i(u^{(t-1)})) =
\frac{-\theta}{n} \sum_{i=1}^n 1[i\in I]
    (\alpha_i^{(t-1)} + \nabla \phi_i(u^{(t-1)})) ~. 
\]
Taking expectation w.r.t. the choice of $I$ and noting that $\tE[ 1[i
\in I]] = m/n$ we obtain that
\[
\tE[ \Delta \bar{\alpha}^{(t)}] = \frac{-\theta m}{n^2} \sum_{i=1}^n 
    (\alpha_i^{(t-1)} + \nabla \phi_i(u^{(t-1)})) =  \frac{-\theta
      m}{n} \left( \bar{\alpha}^{(t-1)} + \nabla f(u^{(t-1)})\right)~. 
\]
\end{proof}

Next, we upper bound the ``variance'' of $\Delta \bar{\alpha}^{(t)}$,
in the sense of the expected squared norm of the difference between
$\Delta \bar{\alpha}^{(t)}$ and its expectation.
\begin{lemma} \label{lem:var}
At each round $t$, we have
\[
\tE \|\Delta \bar{\alpha}^{(t)}-\tE \Delta \bar{\alpha}^{(t)}\|^2 
\leq \frac{m \theta^2}{n^3} \sum_{i=1}^n \|\alpha_i^{(t-1)} + \nabla \phi_i (u^{(t-1)})\|^2 .
\]
\end{lemma}
\begin{proof}
  We introduce the simplified notation 
  $\beta_i =-\theta(\alpha_{i}^{(t-1)} + \nabla \phi_i (u^{(t-1)}))$ and $\mu= \frac{n}{m} \tE \Delta \bar{\alpha}^{(t)}$.
  Note that $\beta_i$ is independent of the choice of $I$ (thus can be considered as a deterministic number).
  Then $\beta_i = \alpha_i^{(t)} - \alpha_{i}^{(t-1)}$ when $i \in I$, and $n^{-1} \sum_{i=1}^n \beta_i = \mu$.
  We thus have
\begin{align*}
  &\tE \|\Delta \bar{\alpha}^{(t)}-\tE \Delta \bar{\alpha}^{(t)}\|^2 
  = \tE \|n^{-1} \sum_{i \in I} (\beta_i -\mu)\|^2 \\
  =& n^{-2} \tE \sum_{i,j \in I} (\beta_i -\mu)^\top(\beta_j -\mu) \\
  =& n^{-2} \tE \sum_{i \in I} \|\beta_i -\mu\|^2 + n^{-2} \tE \sum_{i \neq j \in I} (\beta_i -\mu)^\top(\beta_j -\mu) .
\end{align*}
Note that for any $i \neq j \in I$: $\beta_i-\mu$ and $\beta_j-\mu$
can be regarded as zero-mean random vectors that are drawn uniformly
at random from the same distribution without replacement.  Therefore
they are not positively correlated when $i \neq j$. That is, we have
\[
  \tE \sum_{i \neq j \in I} (\beta_i -\mu)^\top (\beta_j -\mu) \leq 0 .
\]
Therefore
\begin{align*}
  \tE \|\Delta \bar{\alpha}^{(t)}-\tE \Delta \bar{\alpha}^{(t)}\|^2 
\leq&  \frac{1}{n^2} \tE \sum_{i \in I} \|\beta_i -\mu\|^2 
=  \frac{m}{n^{3}} \sum_{i=1}^n \|\beta_i-\mu\|^2 \\
\leq&  \frac{m}{n^{3}} \sum_{i=1}^n \|\beta_i\|^2 
= \frac{m \theta^2}{n^3} \sum_{i=1}^n \|\alpha_i^{(t-1)} + \nabla \phi_i (u^{(t-1)})\|^2 .
\end{align*}
\end{proof}

Recall that the theorem upper bounds the expected value of $m \Delta
P(x^{(t)}) + n \Delta D(\alpha^{(t)})$, which in turns upper bound the
duality gap at round $t$. The following lemma derives an upper bound
on this quantity that depends on the value of this quantity at the
previous iteration and three additional terms. We will later show
that the sum of the additional terms is negative in expectation.
The lemma uses standard algebraic manipulations as well as the
assumptions on $g$ and $\phi_i$.
\begin{lemma} \label{lem:technical}
 For each round $t$ we have
\begin{align*}
\left[m \Delta P(x^{(t)}) + n \Delta D(\alpha^{(t)}) \right] - \left(1-\frac{\theta m}{n}\right) \left[ m \Delta P(x^{(t-1)})+ n \Delta
  D(\alpha^{(t-1)})\right] \le a_I^{(t)} + b_I^{(t)} + c_I^{(t)} ~,
\end{align*}
where
\begin{align*}
  & a_I^{(t)} =
  \frac{m}{2 \gamma} \|x^{(t)} - u^{(t-1)}\|^2 -\frac{\theta (1-\theta)\gamma}{2} \sum_{i \in I} \|\alpha_i^{(t-1)}+\nabla  \phi_i (u^{(t-1)})\|^2 ~~,\\
  & b_I^{(t)} = \theta m f(u^{(t-1)}) + \frac{m \theta}{n}
  \sum_{i=1}^n \phi_i^*(-\alpha^{(t-1)}_i)
  + m \nabla f(u^{(t-1)})^\top (-\theta u^{(t-1)}) \\
  & \quad + \sum_{i \in I} \left[ -\theta \phi_i^*(-\alpha^{(t-1)}_i)+
    \theta [\nabla \phi_i(u^{(t-1)})^\top u^{(t-1)} -
    \phi_i(u^{(t-1)})]
  \right]  ~~, \\
  & c_I^{(t)} = m \nabla f(u^{(t-1)})^\top (x^{(t)} - (1-\theta)
  x^{(t-1)}) + \frac{n + m\theta}{2\lambda}\|\bar{\alpha}^{(t)}\|^2 -
  \frac{n - m\theta }{2\lambda} \|\bar{\alpha}^{(t-1)}\|^2 ~~.
\end{align*}
\end{lemma}
\begin{proof}
Since $m/n \le 1$ we have
\begin{align*}
&\left[m \Delta P(x^{(t)}) + n \Delta D(\alpha^{(t)}) \right] - \left(1-\frac{\theta m}{n}\right) \left[ m \Delta P(x^{(t-1)})+ n \Delta
  D(\alpha^{(t-1)})\right] \\
&\le \left[m \Delta P(x^{(t)}) + n \Delta D(\alpha^{(t)}) \right] -
\left[ (m-\theta m) \Delta P(x^{(t-1)}) + (n-\theta m) \Delta
  D(\alpha^{(t-1)})\right] \\
&=m[\Delta P(x^{(t)})-\Delta P(x^{(t-1)})] + n [\Delta
  D(\alpha^{(t)}) - \Delta D(\alpha^{(t-1)})]
  + \theta m[\Delta P(x^{(t-1)}) + \Delta D(\alpha^{(t-1)})] \\ 
&= m[P(x^{(t)})-P(x^{(t-1)})] - n [D(\alpha^{(t)}) - D(\alpha^{(t-1)}) ]
  + \theta m[P(x^{(t-1)}) - D(\alpha^{(t-1)})]  \\
&= m[P(x^{(t)})-(1-\theta)P(x^{(t-1)})] - n [D(\alpha^{(t)}) - D(\alpha^{(t-1)}) ]
  - \theta m D(\alpha^{(t-1)}) ~.
\end{align*}
Therefore, we need to show that the right-hand side of the above is
upper bounded by $a_I^{(t)} + b_I^{(t)} + c_I^{(t)}$.

\paragraph{Step 1:} We first bound $m[P(x^{(t)})-(1-\theta)P(x^{(t-1)})]$.
Using the smoothness of $f$ we have
\begin{align*}
f(x^{(t)})  &\le f(u^{(t-1)}) + \nabla f(u^{(t-1)})^\top (x^{(t)} - u^{(t-1)}) 
+ \frac{1}{2 \gamma} \|x^{(t)} - u^{(t-1)}\|^2 \\
&= (1-\theta)f(u^{(t-1)}) + \theta f(u^{(t-1)}) + \nabla
f(u^{(t-1)})^\top (x^{(t)} - u^{(t-1)}) + \frac{1}{2 \gamma} \|x^{(t)}
- u^{(t-1)}\|^2 
\end{align*}
and using the convexity of $f$ we also have
\[
 (1-\theta)f(u^{(t-1)}) \le 
(1-\theta)\left[f(x^{(t-1)}) - \nabla
  f(u^{(t-1)})^\top(x^{(t-1)}-u^{(t-1)})\right]  ~.
\]
Combining the above two inequalities and rearranging terms we obtain
\begin{align} \label{eqn:fxtb1}
f(x^{(t)})  &\le (1-\theta)f(x^{(t-1)}) + \theta f(u^{(t-1)}) \\ \nonumber
&+ \nabla
f(u^{(t-1)})^\top \left( x^{(t)} - \theta u^{(t-1)} - (1-\theta)x^{(t-1)}\right) + \frac{1}{2 \gamma} \|x^{(t)}
- u^{(t-1)}\|^2 ~.
\end{align}
Next, using the convexity of $g$ we have
\[
g(x^{(t)}) = \frac{\lambda}{2} \| x^{(t)}\|^2 = 
\frac{\lambda}{2} \left\| (1-\theta)
x^{(t-1)} + \frac{\theta}{\lambda} \bar{\alpha}^{(t)}\right\|^2 \le
\frac{\lambda(1-\theta)}{2} \|x^{(t-1)}\|^2 + \frac{\theta}{2\lambda}
\|\bar{\alpha}^{(t)}\|^2 ~.
\]
Combining this with \eqref{eqn:fxtb1} we obtain
\begin{align*}
 P(x^{(t)}) &= f(x^{(t)}) + g(x^{(t)}) \\
&\le (1-\theta)\left( f(x^{(t-1)}) + \tfrac{\lambda}{2}
  \|x^{(t-1)}\|^2\right) +  \theta f(u^{(t-1)}) \\ \nonumber
&+ \nabla
f(u^{(t-1)})^\top \left( x^{(t)} - \theta u^{(t-1)} - (1-\theta)x^{(t-1)}\right) + \frac{1}{2 \gamma} \|x^{(t)}
- u^{(t-1)}\|^2 
+ \frac{\theta}{2\lambda}
\|\bar{\alpha}^{(t)}\|^2 ~,
\end{align*}
which yields
\begin{align} \nonumber
m\left[ P(x^{(t)})-(1-\theta) P(x^{(t-1)})\right] &\le m\theta
f(u^{(t-1)}) + m\,\nabla
f(u^{(t-1)})^\top \left( x^{(t)} - \theta u^{(t-1)} - (1-\theta)x^{(t-1)}\right) \\ \label{eqn:PmPb2}
&+ \frac{m}{2 \gamma} \|x^{(t)}
- u^{(t-1)}\|^2 
+ \frac{\theta m}{2\lambda}
\|\bar{\alpha}^{(t)}\|^2 ~.
\end{align}

\paragraph{Step 2:} Next, we bound $- n [D(\alpha^{(t)}) - D(\alpha^{(t-1)}) ] $.
Using the definition of the dual update we have
\begin{equation} \label{eqn:duinterm1}
- n [D(\alpha^{(t)}) - D(\alpha^{(t-1)}) ] 
= \sum_{i \in I} \left[\phi_i^*(-\alpha^{(t)}_i) -\phi_i^*(-\alpha^{(t-1)}_i)\right] + 
\frac{n}{2\lambda} \left[\|\bar{\alpha}^{(t)}\|^2 -
  \|\bar{\alpha}^{(t-1)}\|^2\right] ~.
\end{equation}
For all $i \in I$, we may use the definition of the update of $\alpha^{(t)}_i$ in the algorithm, 
the strong-convexity of $\phi^*_i$, and the equality in Fenchel-Young for gradients to obtain:
\begin{align*}
&\phi_i^*(-\alpha^{(t)}_i) \le (1-\theta) \phi^*_i(-\alpha^{(t-1)}_i)
+\theta \phi_i^*(\nabla \phi_i(u^{(t-1)})) - \frac{\theta
  (1-\theta)\gamma}{2} \|\alpha^{(t-1)}_i+\nabla \phi_i(u^{(t-1)})\|^2
\\
&=  (1-\theta) \phi^*_i(-\alpha^{(t-1)}_i)
+\theta [\nabla \phi_i(u^{(t-1)})^\top u^{(t-1)} - \phi_i(u^{(t-1)})] - \frac{\theta
  (1-\theta)\gamma}{2} \|\alpha^{(t-1)}_i+\nabla \phi_i(u^{(t-1)})\|^2~.
\end{align*}
Combining this with \eqref{eqn:duinterm1} we get
\begin{align} \label{eqn:dDstep2}
&- n [D(\alpha^{(t)}) - D(\alpha^{(t-1)}) ] \le \\ \nonumber
&\sum_{i \in I} \left[ 
-\theta \phi^*_i(-\alpha^{(t-1)}_i)
+\theta [\nabla \phi_i(u^{(t-1)})^\top u^{(t-1)} - \phi_i(u^{(t-1)})] - \frac{\theta
  (1-\theta)\gamma}{2} \|\alpha^{(t-1)}_i+\nabla \phi_i(u^{(t-1)})\|^2
\right]\\ \nonumber
& + 
\frac{n}{2\lambda} \left[\|\bar{\alpha}^{(t)}\|^2 -
  \|\bar{\alpha}^{(t-1)}\|^2\right] ~.
\end{align}

\paragraph{Step 3:} 
Summing \eqref{eqn:dDstep2}, \eqref{eqn:PmPb2}, and the equation 
\[
-m\theta
D(\alpha^{(t-1)}) = \frac{m\theta}{n} \sum_i
\phi^*_i(-\alpha_i^{(t-1)}) + \frac{m\theta}{2\lambda}
\|\bar{\alpha}^{(t-1)}\|^2
\]
we obtain that
\begin{align*}
&m[P(x^{(t)})-(1-\theta)P(x^{(t-1)})] - n [D(\alpha^{(t)}) - D(\alpha^{(t-1)}) ]
  - \theta m D(\alpha^{(t-1)}) \le \\
& m\theta
f(u^{(t-1)}) + m\,\nabla
f(u^{(t-1)})^\top \left( x^{(t)} - \theta u^{(t-1)} - (1-\theta)x^{(t-1)}\right) \\ 
&+ \frac{m}{2 \gamma} \|x^{(t)}
- u^{(t-1)}\|^2 
+ \frac{\theta m}{2\lambda}
\|\bar{\alpha}^{(t)}\|^2 \\
&+\sum_{i \in I} \left[ 
-\theta \phi^*_i(-\alpha^{(t-1)}_i)
+\theta [\nabla \phi_i(u^{(t-1)})^\top u^{(t-1)} - \phi_i(u^{(t-1)})] - \frac{\theta
  (1-\theta)\gamma}{2} \|\alpha^{(t-1)}_i+\nabla \phi_i(u^{(t-1)})\|^2
\right]\\
& + 
\frac{n}{2\lambda} \left[\|\bar{\alpha}^{(t)}\|^2 -
  \|\bar{\alpha}^{(t-1)}\|^2\right] \\
&+\frac{m\theta}{n} \sum_i
\phi^*_i(-\alpha_i^{(t-1)}) + \frac{m\theta}{2\lambda}
\|\bar{\alpha}^{(t-1)}\|^2 \\
&= a_I^{(t)} + b_I^{(t)} + c_I^{(t)} ~.
\end{align*}

\end{proof}

\begin{lemma} \label{lem:negExpectation}
  At each round $t$, let $a_I^{(t)},b_I^{(t)},c_I^{(t)}$ be as defined
  in \lemref{lem:technical}. Then, 
\[
\tE[a_I^{(t)}+b_I^{(t)}+c_I^{(t)}] \le 0 ~.
\]
\end{lemma}
\begin{proof}
Recall,
\begin{align*}
  & a_I^{(t)} =
  \frac{m}{2 \gamma} \|x^{(t)} - u^{(t-1)}\|^2 -\frac{\theta (1-\theta)\gamma}{2} \sum_{i \in I} \|\alpha_i^{(t-1)}+\nabla  \phi_i (u^{(t-1)})\|^2 ~~,\\
  & b_I^{(t)} = \theta m f(u^{(t-1)}) + \frac{m \theta}{n}
  \sum_{i=1}^n \phi_i^*(-\alpha^{(t-1)}_i)
  + m \nabla f(u^{(t-1)})^\top (-\theta u^{(t-1)}) \\
  & \quad + \sum_{i \in I} \left[ -\theta \phi_i^*(-\alpha^{(t-1)}_i)+
    \theta [\nabla \phi_i(u^{(t-1)})^\top u^{(t-1)} -
    \phi_i(u^{(t-1)})]
  \right]  ~~, \\
  & c_I^{(t)} = m \nabla f(u^{(t-1)})^\top (x^{(t)} - (1-\theta)
  x^{(t-1)}) + \frac{n + m\theta}{2\lambda}\|\bar{\alpha}^{(t)}\|^2 -
  \frac{n - m\theta }{2\lambda} \|\bar{\alpha}^{(t-1)}\|^2 ~~.
\end{align*}

\paragraph{Step 1:} We first show that $\tE[b_I^{(t)}] = 0$. Indeed,
\begin{align*}
&\tE\left[\sum_{i \in I} \left[ -\theta \phi_i^*(-\alpha^{(t-1)}_i)+
    \theta [\nabla \phi_i(u^{(t-1)})^\top u^{(t-1)} -
    \phi_i(u^{(t-1)})]
  \right] \right] \\
&= \tE\left[\sum_{i=1}^n 1[i \in I] \left[ -\theta \phi_i^*(-\alpha^{(t-1)}_i)+
    \theta [\nabla \phi_i(u^{(t-1)})^\top u^{(t-1)} -
    \phi_i(u^{(t-1)})]
  \right] \right] \\
&= \sum_{i=1}^n \left[ -\theta \phi_i^*(-\alpha^{(t-1)}_i)+
    \theta [\nabla \phi_i(u^{(t-1)})^\top u^{(t-1)} -
    \phi_i(u^{(t-1)})]
  \right] \tE\left[ 1[i \in I]\right] \\
&= \frac{m}{n} \sum_{i=1}^n \left[ -\theta \phi_i^*(-\alpha^{(t-1)}_i)+
    \theta [\nabla \phi_i(u^{(t-1)})^\top u^{(t-1)} -
    \phi_i(u^{(t-1)})]
  \right] \\
&= -\left[ \theta m f(u^{(t-1)}) + \frac{m \theta}{n}
  \sum_{i=1}^n \phi_i^*(-\alpha^{(t-1)}_i)
  + m \nabla f(u^{(t-1)})^\top (-\theta u^{(t-1)})  \right] ~.
\end{align*}

\paragraph{Step 2:} We next show that
\[
\tE[c_I^{(t)}] = \frac{n+m\theta}{2\lambda} \tE\left[ \|\Delta
  \bar{\alpha}^{(t)} - \tE \Delta
  \bar{\alpha}^{(t)} \|^2 \right] -  \frac{n-m\theta}{2\lambda}
\left\|\tE \Delta   \bar{\alpha}^{(t)} \right\|^2~.
\]
Indeed, using \lemref{lem:DeltatE} we know that
\[
\nabla f(u^{(t-1)}) = - \bar{\alpha}^{(t-1)} - \frac{n}{\theta\,m}
\tE[\Delta \bar{\alpha}^{(t)}] ~,
\]
and by the definition of the update, 
\[
x^{(t)} - (1-\theta) x^{(t-1)} = \frac{\theta}{\lambda}
\bar{\alpha}^{(t)} ~.
\]
Therefore,
\begin{align*}
  \tE c_I^{(t)} &= m \nabla f(u^{(t-1)})^\top \tE[x^{(t)} - (1-\theta)
  x^{(t-1)}] + \frac{n + m\theta}{2\lambda}
  \tE\|\bar{\alpha}^{(t)}\|^2 -
  \frac{n - m\theta }{2\lambda} \|\bar{\alpha}^{(t-1)}\|^2\\
  &= -m \left[ \bar{\alpha}^{(t-1)} + \frac{n}{\theta\,m} \tE[\Delta
    \bar{\alpha}^{(t)}]\right]^\top \left[\frac{\theta}{\lambda}
    \tE\bar{\alpha}^{(t)} \right] + \frac{n + m\theta}{2\lambda}
  \tE\|\bar{\alpha}^{(t)}\|^2 -
  \frac{n - m\theta }{2\lambda}  \|\bar{\alpha}^{(t-1)}\|^2\\
  &= -m \left[ \bar{\alpha}^{(t-1)} + \frac{n}{\theta\,m} 
    \tE[\Delta \bar{\alpha}^{(t)}]\right]^\top \left[\frac{\theta}{\lambda}
    [\tE \Delta \bar{\alpha}^{(t)}+ \bar{\alpha}^{(t-1)}] \right] \\
  & \quad + \frac{n + m\theta}{2\lambda} 
  \tE\|\Delta \bar{\alpha}^{(t)} + \bar{\alpha}^{(t-1)}\|^2 -
  \frac{n - m\theta }{2\lambda}  \|\bar{\alpha}^{(t-1)}\|^2\\
  &= -\frac{n}{\lambda} \|\tE \Delta \bar{\alpha}^{(t)} \|^2 +
  \frac{n+m\theta}{2\lambda} \tE \|\Delta
  \bar{\alpha}^{(t)}\|^2 \\
  &= \frac{n+m\theta}{2\lambda} \left[ \tE \|\Delta
    \bar{\alpha}^{(t)}\|^2 - \| \tE \Delta \bar{\alpha}^{(t)} \|^2
  \right] - \frac{n-m\theta}{2\lambda}
  \|\tE \Delta   \bar{\alpha}^{(t)} \|^2 \\
  &=\frac{n+m\theta}{2\lambda} \tE\left[ \|\Delta \bar{\alpha}^{(t)} -
    \tE \Delta \bar{\alpha}^{(t)} \|^2 \right] -
  \frac{n-m\theta}{2\lambda} \|\tE \Delta \bar{\alpha}^{(t)} \|^2 ~.
\end{align*}

\paragraph{Step 3:} Next, we upper bound $\tE[a_I^{(t)}]$. 

Using the definitions of $x^{(t)}$ and $u^{(t-1)}$
and the smoothness of $g^*$ we have
\[
\|x^{(t)}- u^{(t-1)}\| = \theta \| \nabla g^*(\bar{\alpha}^{(t)}) -
\nabla g^*(\bar{\alpha}^{(t-1)}) \| = \frac{\theta}{\lambda}  \| \bar{\alpha}^{(t)}-
\bar{\alpha}^{(t-1)} \| ~.
\]
Therefore, 
\[
a_I^{(t)} = \frac{m\,\theta^2}{2\gamma\lambda^2}  \| \Delta
\bar{\alpha}^{(t)} \|^2 - \frac{\theta (1-\theta)\gamma}{2} \sum_{i \in I} \|\alpha_i^{(t-1)}+\nabla  \phi_i (u^{(t-1)})\|^2 ~.
\]
Taking expectation we obtain
\begin{align*}
\tE a_I^{(t)} &= \frac{m\,\theta^2}{2\gamma\lambda^2}  \tE \| \Delta
\bar{\alpha}^{(t)} \|^2 - \frac{\theta (1-\theta)\gamma}{2} \sum_{i
  =1}^n \|\alpha_i^{(t-1)}+\nabla  \phi_i (u^{(t-1)})\|^2 \,\tE[1[i \in
I]] \\
&=\frac{m\,\theta^2}{2\gamma\lambda^2}  \tE \| \Delta
\bar{\alpha}^{(t)} \|^2 - \frac{\theta (1-\theta)\gamma m}{2n} \sum_{i
  =1}^n \|\alpha_i^{(t-1)}+\nabla  \phi_i (u^{(t-1)})\|^2  \\
&\le \frac{m\,\theta^2}{2\gamma\lambda^2}  \tE \| \Delta
\bar{\alpha}^{(t)} \|^2 - \frac{\theta (1-\theta)\gamma m}{2n} \cdot
\frac{n^3}{m \theta^2} \,\tE \|\Delta \bar{\alpha}^{(t)}-\tE \Delta \bar{\alpha}^{(t)}\|^2 ~,
\end{align*}
where the last inequality follows from \lemref{lem:var}. Using the
equality 
\[\tE\|\Delta \bar{\alpha}^{(t)}\|^2 =\tE \|\Delta
\bar{\alpha}^{(t)}-\tE \Delta \bar{\alpha}^{(t)}\|^2 + \|\tE\Delta
\bar{\alpha}^{(t)}\|^2
\]
we obtain that
\[
\tE a_I^{(t)} \le \frac{m\,\theta^2}{2\gamma\lambda^2}\|\tE\Delta
\bar{\alpha}^{(t)}\|^2 - \left(\frac{(1-\theta)\gamma n^2}{2\theta} -
  \frac{m\,\theta^2}{2\gamma\lambda^2}\right) \tE \|\Delta
\bar{\alpha}^{(t)}-\tE \Delta \bar{\alpha}^{(t)}\|^2 ~.
\]

\paragraph{Step 4:} To conclude the proof, we combine the bounds derived
in the three steps above and get that
\begin{align*}
&\tE[a_I^{(t)}+b_I^{(t)}+c_I^{(t)}] \le \\ 
&\left(\frac{m\,\theta^2}{2\gamma\lambda^2} - \frac{n-m\theta}{2\lambda}\right)\|\tE\Delta
\bar{\alpha}^{(t)}\|^2 + \left(\frac{m\,\theta^2}{2\gamma\lambda^2} + \frac{n+m\theta}{2\lambda}-\frac{(1-\theta)\gamma n^2}{2\theta} \right) \tE \|\Delta
\bar{\alpha}^{(t)}-\tE \Delta \bar{\alpha}^{(t)}\|^2 ~.
\end{align*}
A sufficient condition for the above to be non-positive is that 
\begin{align*}
&1.~~m \theta^2 + m \theta \gamma \lambda -  \gamma \lambda n\le 0 \\
&2.~~m \theta^3 + \gamma \lambda \theta(n+m\theta) - (1-\theta)
(\gamma \lambda n)^2 \le 0
\end{align*}
Let us require the stronger conditions
\begin{align*}
&1.1~~m \theta^2 -  \gamma \lambda n/2 \le 0 \Rightarrow \theta \le
\sqrt{0.5\,\gamma \lambda n/(m)}\\
&1.2~~m \theta \gamma \lambda -  \gamma \lambda n/2\le 0  \Rightarrow
\theta \le 0.5 \, (n/m)\\
&2.1~~m \theta^3 - (\gamma \lambda n)^2/4 \le 0 \Rightarrow \theta \le
\left( \frac{(\gamma \lambda n)^2}{4m} \right)^{1/3} \\
&2.2~~\gamma \lambda \theta n - (\gamma \lambda n)^2/4 \le 0 \Rightarrow
\theta \le \gamma \lambda n / 4 \\
&2.3~~ \gamma \lambda \theta^2 m - (\gamma \lambda n)^2/4 \le 0 \Rightarrow
\theta \le \frac{\gamma \lambda n/2}{\sqrt{\gamma \lambda m}} =
0.5~ \sqrt{\gamma \lambda n \cdot \frac{n}{m}} \\
&2.4~~ \theta (\gamma \lambda n)^2- (\gamma \lambda n)^2/4 \le 0 \Rightarrow
\theta \le 1/4
\end{align*}
An even strong condition is
\begin{equation}
\theta \le \frac{1}{4} \min\left\{
1 ~,~ \sqrt{\frac{\gamma \lambda n}{m}} ~,~ \gamma \lambda n ~,~
\frac{(\gamma \lambda n)^{2/3}}{m^{1/3}} \right\} ~.
\end{equation}

\end{proof}

\begin{proof}[{\bf Proof of \thmref{thm:main}}]
Combining \lemref{lem:technical}
and \lemref{lem:negExpectation} yields
\begin{align*}
\tE\left[m \Delta P(x^{(t)}) + n \Delta D(\alpha^{(t)}) \right] \le \left(1-\frac{\theta m}{n}\right) \tE\left[ m \Delta P(x^{(t-1)})+ n \Delta
  D(\alpha^{(t-1)})\right] ~.
\end{align*}
Taking expectation with respect to the randomness in previous rounds,
using the law of total probability, and applying the above inequality
recursively, we conclude our proof.
\end{proof}

\section{Discussion and Related Work} \label{sec:related}

We have introduced an accelerated version of stochastic dual
coordinate ascent with mini-batches. We have shown, both theoretically
and empirically, that the resulting algorithm interpolates between the
vanilla stochastic coordinate descent algorithm and the accelerated
gradient descent algorithm. 

Using mini-batches in stochastic learning has received a lot of
attention in recent years.  E.g. \cite{ShalevSiSr07} reported
experiments showing that applying small mini-batches in Stochastic
Gradient Descent (SGD) decreases the required number of
iterations. \cite{dekel2012optimal} and \cite{agarwal2012distributed}
gave an analysis of SGD with mini-batches for smooth loss
functions. \cite{cotter2011better} studied SGD and accelerated
versions of SGD with mini-batches and \cite{takac2013mini} studied
SDCA with mini-batches for SVMs. \cite{duchi2010distributed} studied
dual averaging in distributed networks as a function of spectral
properties of the underlying graph. However, all of these methods have
a polynomial dependence on $1/\epsilon$, while we consider the
strongly convex and smooth case in which a $\log(1/\epsilon)$ rate is
achievable.\footnote{It should be noted that one can use our results
  for Lipschitz functions as well by smoothing the loss function (see
  \cite{nesterov2005smooth}). By doing so, we can interpolate between
  the $1/\epsilon^2$ rate of non-accelerated method and the
  $1/\epsilon$ rate of accelerated gradient.}

It is interesting to note that most\footnote{There are few exceptions
  in the context of stochastic coordinate descent \emph{in the primal}. See
  for example \cite{bradley2011parallel,richtarik2012parallel}} of these papers focus on
mini-batches as the method of choice for distributing SGD or SDCA, while ignoring
the option to divide the data by features instead of by examples. A
possible reason is the cost of opening communication sockets as
discussed in \secref{sec:parallel}. 

There are various practical considerations that one should take into
account when designing a practical system for distributed
optimization. We refer the reader, for example, to
\cite{dekeldistribution,low2010graphlab,low2012distributed,agarwal2011reliable,niu2011hogwild}.

The more general problem of distributed PAC learning has been studied
recently in \cite{daume2012protocols,balcan2012distributed}. See also
\cite{long2011algorithms}. In particular, they obtain algorithm with
$O(\log(1/\epsilon))$ communication complexity. However, these works
consider efficient algorithms only in the realizable case.

\bibliographystyle{plainnat}
\bibliography{curRefs}

\end{document}